\documentclass[11pt]{article}

\usepackage[utf8]{inputenc}
\usepackage[T1]{fontenc}
\usepackage{lmodern}
\usepackage{amsmath,amssymb,amsthm}
\usepackage{mathrsfs}
\usepackage{booktabs}
\usepackage{graphicx}
\usepackage{xurl}
\usepackage{hyperref}
\usepackage{microtype}
\usepackage[margin=1in]{geometry}
\usepackage{tikz}
\usetikzlibrary{arrows.meta,positioning,calc,fit,backgrounds}

\definecolor{rapidblue}{RGB}{49,91,138}
\definecolor{rapidbluefill}{RGB}{232,240,248}
\definecolor{rapidgreen}{RGB}{55,119,92}
\definecolor{rapidgreenfill}{RGB}{232,245,238}
\definecolor{rapidorange}{RGB}{176,103,34}
\definecolor{rapidorangefill}{RGB}{251,239,224}
\definecolor{rapidgray}{RGB}{92,98,105}
\definecolor{rapidgrayfill}{RGB}{242,243,244}

\newcommand{\method}{RAPID}
\newcommand{\E}{\mathbb{E}}
\newcommand{\Loss}{\mathcal{L}}
\newcommand{\Pairs}{\mathcal{P}}
\newcommand{\sg}{\operatorname{sg}}

\newcommand{\Var}{\operatorname{Var}}
\newcommand{\Cov}{\operatorname{Cov}}
\newcommand{\KL}{\operatorname{KL}}
\newcommand{\eps}{\varepsilon}
\newtheorem{theorem}{Theorem}
\newtheorem{proposition}{Proposition}

\newcounter{algorithm}
\renewcommand{\thealgorithm}{\arabic{algorithm}}
\newenvironment{rapidalgorithm}[2][t]{%
  \refstepcounter{algorithm}%
  \begin{figure}[#1]\centering
  \begin{minipage}{0.94\linewidth}\small
  \textbf{Algorithm \thealgorithm: #2}\par\smallskip\hrule\smallskip
}{%
  \smallskip\hrule
  \end{minipage}\end{figure}
}

\title{RAPID: Reliability-Aware Pair Importance Distillation}
\author{%
  Ali Mahdavi \\
  Department of Computer Engineering\\
  SRC, Islamic Azad University\\
  Tehran, Iran \\
  \texttt{ali.mahdavi@iau.ir} \\
  \and
  Azadeh Zamanifar \\
  Department of Computer Engineering\\
  SRC, Islamic Azad University\\
  Tehran, Iran \\
  \texttt{azamanifar@iau.ac.ir} \\
  \and
  Amir Farhad Farhadi \\
  School of Computer Engineering \\
  Iran University of Science and Technology \\
  Tehran, Iran \\
  \texttt{amfarhadi@mail.iust.ac.ir} \\
  \and
  Omid Kashefi \\
  Meta \\
  CA, USA \\
  \texttt{kashefi@meta.com} \\
}
\date{}

\begin{document}
\maketitle

\begin{abstract}
Inter-example relational distillation transfers a teacher's representation geometry by matching relations between examples in a mini-batch.  Computing all pairs costs $O(B^2)$, while uniform subsampling may spend a limited relation budget inefficiently.  We introduce Reliability-Aware Pair Importance Distillation (\method{}), which separates a reliability-gated relational target from a full-support adaptive pair proposal.  Reliability controls which teacher relations are emphasized; calibrated teacher entropy and detached student-teacher residuals control which relations are evaluated.  Exact inverse-proposal correction makes the loss and gradient estimators conditionally unbiased for the gated mini-batch target.  We evaluate two text-classification settings: AG News (BERT-to-DistilBERT, three paired seeds, $K=256$) and SST-2 (DistilBERT-to-DistilBERT, three paired seeds, $K=64$).  Reliability-gated relational distillation has the highest observed student mean on both datasets ($94.285\pm0.054$\% on AG News; $88.800\pm0.532$\% on SST-2), with \method{} second ($94.241\pm0.025$\% and $88.685\pm0.462$\%) and the cross-entropy student at $94.154\pm0.124$\% and $87.271\pm0.162$\%.  Pilot evaluations are charged to the same total budget as main relations.  Across both settings the gated target leads the means while the adaptive proposal stays within seed noise, supporting the modular view that target reliability and evaluation priority are separable design dimensions.
\end{abstract}

\section{Introduction}
Knowledge distillation compresses a teacher by transferring predictive or representational structure to a smaller student \cite{hinton2015distilling}.  For pretrained language models, intermediate-state and relational objectives complement output-level distillation \cite{sun2019patient,jiao2020tinybert,wang2021minilmv2}.  We study \emph{inter-example} relations: the geometry among sentence representations in a mini-batch.  Such geometry can express similarities that independent per-example targets do not capture.

Relational transfer creates a budget problem.  A batch of $B$ examples contains $M=B(B-1)/2$ unordered relations.  Larger effective batches, multiple represented layers, and repeated relation objectives make this quadratic term increasingly costly.  Uniform subsampling controls relation evaluations but ignores that pair losses are heterogeneous.  Conversely, hard-pair mining can alter the intended objective when sampled losses are not corrected.

The heterogeneity raises two decisions that should not be conflated.  First, \emph{target reliability} asks how strongly a teacher relation should contribute to supervision.  A relation involving an example on which the teacher disagrees with the observed label may warrant less weight.  Second, \emph{evaluation priority} asks how often a relation should be evaluated to estimate that target under a limited budget.  A relation with a large student-teacher residual may be valuable to sample even when its target contribution remains reliability weighted.

We separate these roles (Figure~\ref{fig:rapid-pipeline}).  A normalized, label-aware reliability gate defines the target objective.  A distinct proposal begins from calibrated teacher entropy and reliability and gradually incorporates detached student-teacher relational residuals estimated from pilot relations.  Endpoint- and pair-level defensive mixtures provide full support.  Exact inverse-proposal correction then recovers the reliability-gated all-pairs mini-batch objective in expectation.  Throughout the paper, \emph{influence} denotes estimated priority for contributing to the current relational objective; it is not a classical influence-function approximation.

Our contributions are:
\begin{itemize}
    \item a normalized reliability-gated objective for inter-example relational language-model distillation;
    \item a full-support factorized pair proposal that adapts from teacher-side information to student-teacher residuals under a fixed total relation budget;
    \item conditional unbiasedness and bounded-correction guarantees for the loss and gradient estimators; and
    \item an NLP-first protocol that separates checkpoint selection from held-out reporting, evaluated on AG News and SST-2 with three paired seeds each.
\end{itemize}

\paragraph{Summary of advantages.}
Relative to existing relational and importance-sampled distillation methods,
\method{} offers four concrete advantages.
\begin{enumerate}
    \item \textbf{Target--proposal separation.}  Existing relational objectives
    such as RKD \cite{park2019rkd} and CRD \cite{tian2020crd} evaluate every
    pair and do not separate what to trust from what to evaluate;
    importance-sampled distillation \cite{li2018dis} treats confidence and
    sampling probability as the same quantity.  \method{} keeps the
    reliability-weighted target and the adaptive proposal as independent
    objects, so the gate's role is preserved no matter how the proposal
    changes between epochs.
    \item \textbf{Bias-controlled correction with bounded weights.}  The
    unconditional estimator of the relational loss under adaptive sampling is
    biased unless the inverse proposal is applied.  \method{} includes an
    exact inverse-proposal correction and proves conditional unbiasedness
    (Theorems~1--2); the pair-level defensive mixture further bounds the
    inverse weights by $1/\epsilon$, so no realization can dominate the
    gradient (Proposition~3).
    \item \textbf{Full-support proposal under a single fixed budget.}  The
    pair-level defensive mixture guarantees $q_t(i,j)\geq\epsilon/M$ for
    every batch relation, so every pair remains reachable even as the
    proposal adapts.  Pilot evaluations are charged to the same total budget
    $K=\min\{M,R,C_{\mathrm{cap}}\}$ as main evaluations; adaptive
    information does not buy extra relation evaluations.
    \item \textbf{Better engineering trade-offs on the reported workload.}
    On AG News at matched budget, \method{} records the lowest
    student-training time and the lowest peak GPU memory among the five
    distillation methods (62.4~min and 2127~MB, versus 86--131~min and
    2574--2783~MB for the other distillation methods) and stays within seed
    noise of the other relational methods on held-out accuracy.  On SST-2
    it lands within 0.115 points of the top held-out mean at $K=64$ with
    near-identical peak memory.
\end{enumerate}
None of these advantages is asserted in isolation: they are conditional on
the experimental protocol (paired seeds, selection-separated reporting, fixed
budget, single GPU) described in the Experiments section below.

\begin{figure*}[t]
\centering
\resizebox{0.95\textwidth}{!}{%
\begin{tikzpicture}[
  font=\small, >=Latex,
  flow/.style={-{Latex[length=2mm]},draw=rapidgray,line width=.6pt},
  targetflow/.style={-{Latex[length=2mm]},draw=rapidblue,line width=.7pt},
  proposalflow/.style={-{Latex[length=2mm]},draw=rapidgreen,line width=.7pt},
  box/.style={draw=rapidgray,rounded corners=2pt,align=center,
              inner xsep=5pt,inner ysep=5pt,fill=white},
  target/.style={box,draw=rapidblue,fill=rapidbluefill},
  proposal/.style={box,draw=rapidgreen,fill=rapidgreenfill},
  estimate/.style={box,draw=rapidorange,fill=rapidorangefill}
]
\node[box,text width=2.2cm] (batch) at (0,0)
  {mini-batch\\$\mathcal B=\{(x_i,y_i)\}_{i=1}^B$};
\node[box,text width=2.5cm] (signals) at (3.2,0)
  {teacher logits $z_T$\\representations $h_T,h_S$};
\draw[flow] (batch)--(signals);

\node[target,text width=3.2cm] (reliability) at (7.0,1.5)
  {\textbf{target reliability}\\$r_i=\max\{r_{\min},\sqrt{p_T^{\mathrm{cal}}(y_i)m_i}\}$\\$r_i\rightarrow w_{ij}^{(\lambda)}$};
\node[target,text width=3.0cm] (estimand) at (11.2,1.5)
  {gated all-pairs target\\$\displaystyle \Loss_\lambda=\frac1M\sum_{i<j}w_{ij}^{(\lambda)}\ell_{ij}$};
\draw[targetflow] (signals.north east) to[out=25,in=180] (reliability.west);
\draw[targetflow] (reliability)--(estimand);

\node[proposal,text width=3.2cm] (priority) at (7.0,-1.5)
  {\textbf{evaluation priority}\\$s_i^{\mathrm{static}}\propto u_i^\alpha r_i^\beta$\\$s_i^{\mathrm{res}}\propto r_i\bar d_i^{\sg}$};
\node[proposal,text width=3.0cm] (defense) at (11.2,-1.5)
  {full-support proposal\\$q_t(i,j)=(1-\epsilon)q_a(i,j)+\epsilon/M$};
\draw[proposalflow] (signals.south east) to[out=-25,in=180] (priority.west);
\draw[proposalflow] (priority)--(defense);

\node[estimate,text width=3.4cm] (corrected) at (15.2,0)
  {corrected main-sample estimate\\$\displaystyle
  \widehat\Loss_\lambda=\frac1{K_m}\sum_k
  \frac{w_{I_kJ_k}^{(\lambda)}\ell_{I_kJ_k}}
       {M q_t(I_k,J_k)}$};
\draw[proposalflow] (defense.east) to[out=0,in=-120] (corrected.south west);
\draw[targetflow,densely dashed] (estimand.east) to[out=0,in=120] (corrected.north west);
\end{tikzpicture}%
}
\caption{Influence-guided approximate distillation separates the relational
target (blue) from the proposal used to evaluate it (green).  Detached pilot
residuals change sampling priority, not the estimand; inverse-proposal
correction maps sampled main relations back to the reliability-gated target.}
\label{fig:rapid-pipeline}
\end{figure*}
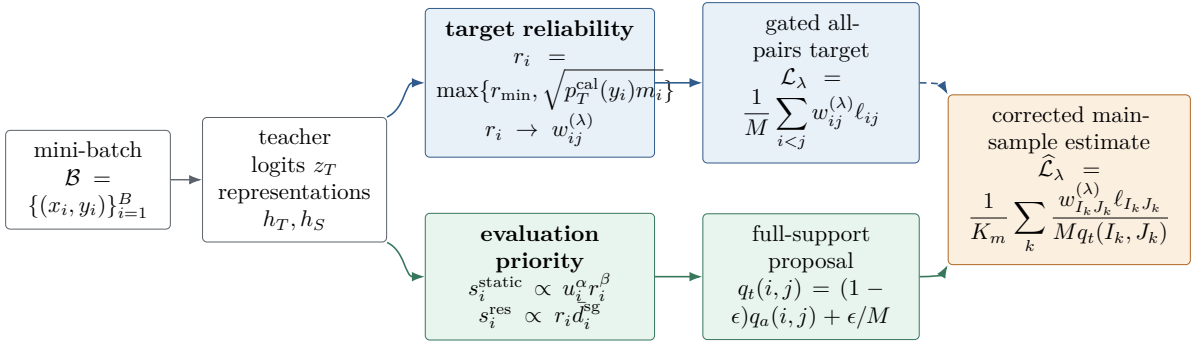

\section{Related Work}
\paragraph{Language-model distillation.}
Patient knowledge distillation transfers intermediate Transformer layers \cite{sun2019patient}.  TinyBERT combines embedding, hidden-state, attention, and prediction objectives \cite{jiao2020tinybert}.  MiniLMv2 transfers self-attention relations \cite{wang2021minilmv2}.  These results show that structural transfer helps.

\paragraph{Relational and importance-sampled distillation.}
RKD transfers pairwise distances and angles \cite{park2019rkd}.  CRD aligns teacher and student representations contrastively \cite{tian2020crd}.  Recent work continues this line with adaptive multi-teacher weighting combined with inter-sample similarity-matrix transfer \cite{li2026amrd}.  Importance sampling reduces stochastic-gradient variance when training contributions are heterogeneous \cite{katharopoulos2018notall}.  Dynamic importance sampling has been applied to output classes in knowledge distillation \cite{li2018dis}.  Recent NLP work such as TAKE uses influence functions to score training samples for text-classification dataset distillation \cite{vo2026take}.  Task-agnostic influence estimation across LM pretraining is an active thread \cite{nishida2026measuring}.  Broader KD surveys now organize the field along methodological and application axes \cite{degibert2026kd4mt}.  On-policy distillation has been formalized for large language models \cite{song2026opd}.

\paragraph{Teacher calibration.}
Teacher confidence varies across examples and is often miscalibrated \cite{guo2017calibration}.  Recent work extends this line by tying calibration to training-time margin or curvature \cite{morosini2026toosharp} and by feature-modulated label smoothing \cite{bahavan2026fedlas}.

\paragraph{Positioning.}
Our focus is orthogonal to the above.  We define a reliability-weighted inter-example target and estimate it with a corrected adaptive proposal.  Our gate is deliberately label aware.  It combines the calibrated probability assigned to the observed training label with the teacher's predictive margin.  Entropy plays a separate role in the proposal, where uncertainty can indicate an informative boundary example.  This separation prevents uncertainty from simultaneously meaning ``sample more often'' and ``trust more strongly.''  The gate is a supervised weighting signal.  It is not a claim that confidence is a calibrated probability of teacher correctness.
\section{Method}
\subsection{Mini-Batch Relational Target}
For a labeled mini-batch $\mathcal{B}=\{(x_i,y_i)\}_{i=1}^B$, let $h_T^i$ and $h_S^i(\theta)$ be mask-aware mean-pooled teacher and student representations.  To specify zero-vector behavior, define
\begin{equation}
 \nu(v)=\frac{v}{\max(\lVert v\rVert_2,\eps_{\mathrm{norm}})},
 \qquad \eps_{\mathrm{norm}}=10^{-12}.
\end{equation}
The implementation-level cosine relations and squared discrepancy are
\begin{equation}
 c_{ij}^{T}=\langle\nu(h_T^i),\nu(h_T^j)\rangle,\qquad
 c_{ij}^{S}=\langle\nu(h_S^i),\nu(h_S^j)\rangle,
 \quad
 \ell_{ij}=(c_{ij}^{T}-c_{ij}^{S})^2.
\end{equation}
The unordered relation set is $\Pairs_{\mathcal B}=\{(i,j):1\leq i<j\leq B\}$ with $M=B(B-1)/2$.  Teacher quantities are detached.  For $B<2$, the relational term is defined as zero.

\subsection{Reliability-Gated Objective}
We fit a positive scalar temperature $T_{\mathrm{cal}}$ on the internal selection split by minimizing teacher negative log-likelihood.  This calibration temperature is separate from the KD temperature.  Let $p_T^{\mathrm{cal}}=\mathrm{softmax}(z_T/T_{\mathrm{cal}})$, and let $m_i$ be its top-one/top-two probability margin.  We use
\begin{equation}
 r_i=\max\left(r_{\min},\sqrt{p_T^{\mathrm{cal}}(y_i\mid x_i)m_i}\right).
\end{equation}
The gate is label-aware and is therefore a supervised training signal, not a calibrated probability that the teacher is correct.  To change relative trust without silently changing the global relational coefficient, we normalize pair weights:
\begin{align}
 \bar w &= \frac{1}{M}\sum_{i<j}r_i r_j
 =\frac{(\sum_i r_i)^2-\sum_i r_i^2}{B(B-1)},\\
 \widetilde w_{ij}&=\frac{r_i r_j}{\bar w},\qquad
 w_{ij}^{(\lambda)}=(1-\lambda)+\lambda\widetilde w_{ij},
 \qquad \lambda\in[0,1],\\
 \Loss_{\lambda}(\theta;\mathcal B)
 &=\frac{1}{M}\sum_{i<j}w_{ij}^{(\lambda)}\ell_{ij}(\theta).
\end{align}
Reliability is part of the estimand.  It is not an inverse-proposal weight.  The mixture has exact endpoints: $\lambda=0$ is uniform relational KD, $\lambda=1$ is the fully gated target, and the reported configuration uses $\lambda=0.5$ to partially gate unreliable relations without discarding the uniform objective.

\begin{proposition}[Normalized scale-preserving gate]
For $B\geq2$ and positive reliability scores,
\begin{equation}
 \frac1M\sum_{i<j}\widetilde w_{ij}=1,
 \qquad
 \frac1M\sum_{i<j}w_{ij}^{(\lambda)}=1.
\end{equation}
Moreover, multiplying every $r_i$ by a common $\gamma>0$ leaves
$\widetilde w_{ij}$, $w_{ij}^{(\lambda)}$, and $\Loss_\lambda$ unchanged.
\end{proposition}
\begin{proof}
The first identity follows directly from the definition of $\bar w$; the
second follows by convex combination.  Under $r_i\mapsto\gamma r_i$, both
$r_i r_j$ and $\bar w$ scale by $\gamma^2$ and cancel.
\end{proof}

\subsection{Static and Adaptive Endpoint Scores}
We use normalized calibrated predictive entropy $u_i=H(p_T^{\mathrm{cal}}(\cdot\mid x_i))/\log C$ as the default static informativeness signal.  The static endpoint distribution is proportional to
\begin{equation}
 s_i^{\mathrm{static}}=\max(u_i,\varepsilon_u)^{\alpha}
 \max(r_i,r_{\min})^{\beta},
 \qquad
 a_i^{\mathrm{static}}=
 \frac{s_i^{\mathrm{static}}}{\sum_j s_j^{\mathrm{static}}}.
\end{equation}
Including reliability in the proposal is compatible with its distinct role in the target: the variance-optimal proposal depends on the complete gated integrand.

After two warm-up epochs, we linearly increase the residual mixture to at most $\tau_{\max}=0.5$.  Whenever $\tau_t>0$, we spend part of the total relation budget on IID uniform pilot relations $P_k=(I_k^p,J_k^p)$.  Their detached residuals are
\begin{equation}
 d_{P_k}=|c_{I_k^pJ_k^p}^{T}-c_{I_k^pJ_k^p}^{S,\sg}|.
\end{equation}
Let $N_i^p=\sum_k\mathbf1\{i\in P_k\}$ and
$D_i^p=\sum_k\mathbf1\{i\in P_k\}d_{P_k}$.  The endpoint residual estimate
\begin{equation}
 \bar d_i=\frac{D_i^p}{\max(N_i^p,1)}
\end{equation}
counts repeated pilot draws repeatedly and assigns zero before flooring to an
endpoint absent from the pilot.  We form
\begin{equation}
 s_i^{\mathrm{res}}=\max(r_i,r_{\min})\max(\bar d_i,\varepsilon_d),
 \qquad
 a_i^{\mathrm{res}}=\frac{s_i^{\mathrm{res}}}{\sum_j s_j^{\mathrm{res}}},
\end{equation}
and interpolate the normalized endpoint distributions:
\begin{equation}
 a_i^{(t)}=(1-\tau_t)a_i^{\mathrm{static}}+\tau_t a_i^{\mathrm{res}}.
\end{equation}
For $E$ epochs and $W=\min(W_0,E-1)$ warm-up epochs, the schedule is
\begin{equation}
 \tau_t=\begin{cases}
 0,&t\leq W,\\[1mm]
 \tau_{\max}\dfrac{t-W}{E-W},&W<t\leq E.
 \end{cases}
\end{equation}
The reported configuration uses $W_0=2$, $\tau_{\max}=0.5$, and
$\alpha=\beta=1$.  Pilot estimates affect only proposal efficiency; they do
not define the target.

\subsection{Full-Support Pair Proposal}
Draw two endpoints independently from $a$ and reject equal endpoints.  The induced probability of unordered pair $(i,j)$ is
\begin{equation}
 Z(a)=1-\sum_k a_k^2,
 \qquad
 q_a(i,j)=\frac{2a_i a_j}{Z(a)}.
\end{equation}
Here $Z(a)$ is the acceptance probability of one ordered endpoint draw.  The
two possible orders produce the factor of two.
Before pair construction we apply a small endpoint-level defensive mixture,
$\tilde a_i=(1-\epsilon_e)a_i+\epsilon_e/B$.  We then use the pair-level
defensive mixture
\begin{equation}
 q_t(i,j)=(1-\epsilon)q_{\tilde a^{(t)}}(i,j)+\epsilon/M,
\end{equation}
which guarantees support over every batch relation.  After spending $K_p$ pilot evaluations, we draw $K_m=K-K_p$ main relations with replacement.  Integer allocation makes the active pilot share approximately 10\%; batches with fewer than $K$ available relations use all relations directly.  The estimator is
\begin{equation}
 \widehat{\Loss}_{\lambda}=
 \frac{1}{K_m}\sum_{k=1}^{K_m}
 \frac{w_{I_kJ_k}^{(\lambda)}\ell_{I_kJ_k}}
 {M q_t(I_k,J_k)}.
\end{equation}

\begin{proposition}[Proposal normalization]
For any endpoint distribution $a$ with $Z(a)>0$, the factorized pair
probabilities are normalized: $\sum_{i<j}q_a(i,j)=1$.
\end{proposition}
\begin{proof}
$2\sum_{i<j}a_i a_j=(\sum_i a_i)^2-\sum_i a_i^2=1-\sum_i a_i^2$.
\end{proof}

Let $\mathscr F_t$ contain the current mini-batch and labels, realized teacher
outputs, current parameters, detached student representations used to build
the proposal, the pilot sample, and all resulting endpoint scores.  Conditional
on $\mathscr F_t$, $q_t$ is fixed and the main relations are IID draws from it.

\begin{theorem}[Conditional loss unbiasedness]
Assume $K_m\geq1$, clipping is disabled, and the numerical probability floor
is inactive.  Then
\begin{equation}
 \E[\widehat{\Loss}_{\lambda}\mid\mathscr F_t]
 =\Loss_{\lambda}(\theta;\mathcal B).
\end{equation}
\end{theorem}
\begin{proof}
Writing $f_{ij}=w_{ij}^{(\lambda)}\ell_{ij}$, linearity and IID sampling give
\begin{equation}
 \E[\widehat\Loss_\lambda\mid\mathscr F_t]
 =\sum_{i<j}q_t(i,j)\frac{f_{ij}}{M q_t(i,j)}
 =\frac1M\sum_{i<j}f_{ij}=\Loss_\lambda.
\end{equation}
\end{proof}

\begin{theorem}[Conditional stopped-proposal gradient unbiasedness]
Suppose reliability, pilot residuals, proposal probabilities, and sampled
indices are constants for automatic differentiation, as in the
implementation.  At differentiable parameter values,
\begin{equation}
 \E[\nabla_\theta^{\sg}\widehat{\Loss}_{\lambda}\mid\mathscr F_t]
 =\nabla_\theta\Loss_{\lambda}.
\end{equation}
\end{theorem}
\begin{proof}
Conditioned on $\mathscr F_t$, differentiate only $f_{ij}(\theta)$ and apply
the same cancellation pairwise.  The finite sum permits termwise
differentiation.  This is the stopped-proposal autograd gradient; it does not
differentiate through categorical sampling or add a score-function term.
\end{proof}

\begin{proposition}[Bounded correction]
The defensive mixture gives $1/(Mq_t(i,j))\leq1/\epsilon$ for every pair.
\end{proposition}

\subsection{Variance and the Ideal Proposal}
Conditional unbiasedness specifies the target but does not guarantee lower
variance.  For IID main draws, the scalar conditional variance is
\begin{equation}
 \Var(\widehat\Loss_\lambda\mid\mathscr F_t)
 =\frac1{K_m}\left[
 \frac1{M^2}\sum_{i<j}\frac{f_{ij}^2}{q_t(i,j)}
 -\Loss_\lambda^2\right].
 \label{eq:variance}
\end{equation}
The pilot affects this variance through both the realized proposal and the
reduction from total budget $K$ to main budget $K_m$.

\begin{proposition}[Unrestricted scalar optimum]
If $F=\sum_{i<j}f_{ij}>0$, the proposal minimizing
Equation~\eqref{eq:variance} over unrestricted pair distributions is
\begin{equation}
 q_{\mathrm{loss}}^\star(i,j)=\frac{f_{ij}}{F}
 =\frac{w_{ij}^{(\lambda)}\ell_{ij}}
 {\sum_{u<v}w_{uv}^{(\lambda)}\ell_{uv}}.
\end{equation}
\end{proposition}
\begin{proof}
Cauchy--Schwarz gives
$(\sum f_{ij})^2\leq(\sum f_{ij}^2/q_{ij})(\sum q_{ij})$,
with equality exactly when $q_{ij}\propto f_{ij}$.
\end{proof}
The ideal depends on unavailable current pair losses.  RAPID instead uses an
$O(B)$-state endpoint proxy built from entropy, reliability, and sparsely
observed absolute residuals.  Since $\ell_{ij}=d_{ij}^2$ while the residual
score uses $d_{ij}$, the adaptive proposal is a tractable proxy rather than an
optimizer of Equation~\eqref{eq:variance}.

\begin{proposition}[Endpoint-factorization restriction]
For any four distinct endpoints $i,j,k,l$, a factorized proposal satisfies
\begin{equation}
 q_a(i,j)q_a(k,l)=q_a(i,k)q_a(j,l)=q_a(i,l)q_a(j,k).
\end{equation}
Thus, for $B\geq4$, the factorized family cannot represent an arbitrary pair
distribution.
\end{proposition}
\begin{proof}
Each product equals $(2/Z(a))^2a_i a_j a_k a_l$.
\end{proof}
The restriction trades proposal expressivity for linear endpoint state.  The
defensive mixtures further trade variance optimality for support and bounded
corrections.

\subsection{Exact Budget Accounting}
Let $R$ be the requested relation count and $C_{\mathrm{cap}}$ the safety cap.
The effective upper budget for a batch is
\begin{equation}
 K=\min\{M,R,C_{\mathrm{cap}}\}.
\end{equation}
When adaptation is active, pilot fraction $\rho$ is converted to an integer by
\begin{align}
 K_p&=\min\!\left\{
 \max\!\left(1,\left\lfloor\rho K+\tfrac12\right\rfloor\right),
 \max(K-1,0)\right\},\\
 K_m&=K-K_p.
\end{align}
During warm-up, $K_p=0$ and $K_m=K$.  If $M\geq1$ and $K=M$, the sampler bypasses pilots,
enumerates every available pair once, sets $q(i,j)=1/M$, and recovers the exact
relational objective.  Hence $K$ is an upper evaluation budget, not a required
count for undersized final batches.  Figure~\ref{fig:rapid-budget} summarizes
the schedule and accounting.

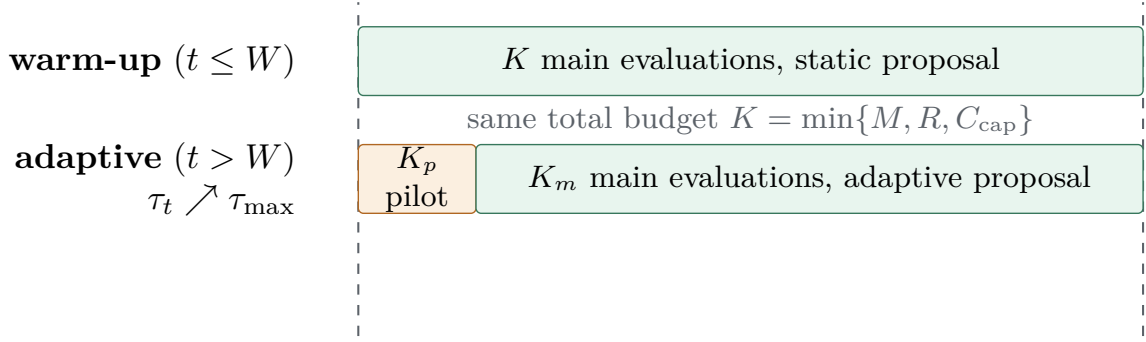
\begin{figure*}[t]
\centering
\resizebox{0.92\textwidth}{!}{%
\begin{tikzpicture}[font=\small,
  bar/.style={rounded corners=1.5pt}]
\draw[rapidgray,dashed,line width=.6pt] (0.5,-1.1) -- (0.5,2.3);
\draw[rapidgray,dashed,line width=.6pt] (8.5,-1.1) -- (8.5,2.3);

\node[align=right,anchor=east] at (0,1.7) {\textbf{warm-up} ($t\leq W$)};
\node[align=right,anchor=east] at (0,0.5)
  {\textbf{adaptive} ($t>W$)\\$\tau_t\nearrow\tau_{\max}$};

\draw[bar,fill=rapidgreenfill,draw=rapidgreen] (0.5,1.35) rectangle (8.5,2.05);
\node[font=\footnotesize] at (4.5,1.7)
  {$K$ main evaluations, static proposal};

\draw[bar,fill=rapidorangefill,draw=rapidorange] (0.5,0.15) rectangle (1.7,0.85);
\draw[bar,fill=rapidgreenfill,draw=rapidgreen] (1.7,0.15) rectangle (8.5,0.85);
\node[align=center,font=\footnotesize] at (1.1,0.5) {$K_p$\\pilot};
\node[font=\footnotesize] at (5.1,0.5)
  {$K_m$ main evaluations, adaptive proposal};

\node[font=\footnotesize,rapidgray] at (4.5,1.1)
  {same total budget $K=\min\{M,R,C_{\mathrm{cap}}\}$};
\end{tikzpicture}%
}
\caption{The total relation budget $K=\min\{M,R,C_{\mathrm{cap}}\}$ includes pilot
evaluations.  Warm-up spends the entire budget on main draws from the static
proposal; adaptation reallocates part of the same total to uniform pilots and
draws the remaining main relations from the interpolated adaptive proposal.
If $M\leq\min(R,C_{\mathrm{cap}})$, the sampler enumerates all $M$ pairs exactly
instead of sampling.}
\label{fig:rapid-budget}
\end{figure*}

\subsection{Training Objective and Complexity}
For classification, the final objective is deliberately minimal:
\begin{equation}
 \Loss=\Loss_{\mathrm{CE}}
 +\alpha_{\mathrm{KD}}T_{\mathrm{KD}}^2
 \KL\!\left(p_T^{(T_{\mathrm{KD}})}\middle\|p_S^{(T_{\mathrm{KD}})}\right)
 +\beta_{\mathrm{rel}}\widehat{\Loss}_{\lambda}.
\end{equation}
RAPID uses the same temperature-scaled KD term as the KD and relational
controls; direct feature MSE is excluded.  Given already-computed
representations, all-pairs similarities cost $O(B^2d)$, whereas pilot and main
similarities cost $O(Kd)$, representation normalization costs $O(Bd)$, and
entropy/reliability scoring costs $O(BC)$.  Transformer forward and backward
passes are unchanged, and rejection sampling has an acceptance-dependent
constant.  Relation-level savings therefore need not reduce end-to-end time.

\paragraph{Training procedure.}
For each mini-batch, the per-step compute and memory costs are:
(1) teacher and student forward passes and the cross-entropy + KL terms,
dominated by the Transformer forward/backward ($O(BC\cdot d)$);
(2) reliability, entropy, and static endpoint scoring over $B$ examples
($O(BC)$);
(3) uniform pilot sampling, detached student--teacher cosine on $K_p$ pairs
($O(K_p d)$);
(4) residual aggregation and the adaptive endpoint scores
($O(B + K_p)$);
(5) factorized pair proposal construction with the endpoint and pair
defensive mixtures ($O(B^2)$);
(6) main sampling of $K_m$ unordered pairs and the corrected relational
loss on those pairs ($O(K_m d)$);
(7) backward through the sampled relational losses
($O(K_m d)$);
(8) the Adam-style update on $\theta$.
Steps (5) and (6) are the only algorithm-specific costs; on $B=32$ the
$B^2=1024$ pair construction is two orders of magnitude smaller than the
Transformer forward over the batch, so the overhead from adaptive sampling is
dominated by Python and CUDA launch latency rather than FLOPs.  The
proposal-side student values are stopped, while the sampled relational
losses retain gradients through student representations.
Algorithm~\ref{alg:rapid-update} specifies the complete update.

\begin{rapidalgorithm}[t]{RAPID mini-batch update}
\label{alg:rapid-update}
\begin{tabular}{@{}r p{0.86\linewidth}@{}}
\textbf{Input:} & mini-batch $\mathcal B$, teacher $T$, student $S_\theta$,
total budget $K$, epoch $t$, gate strength $\lambda$. \\
1. & Compute teacher logits and mask-aware teacher/student representations;
stop gradients through all teacher quantities. \\
2. & Compute calibrated reliability $r_i$, normalized gated pair weights
$w_{ij}^{(\lambda)}$, and static endpoint probabilities
$a^{\mathrm{static}}$. \\
3. & If $K=M$, enumerate all unordered pairs and set $q(i,j)=1/M$; otherwise
continue with sampled relations. \\
4. & During warm-up, set $K_p=0$.  After warm-up, allocate $K_p$ uniform
pilots, evaluate detached residuals, and construct $a^{\mathrm{res}}$. \\
5. & Interpolate endpoint distributions using $\tau_t$, apply endpoint and
pair defensive mixtures, and retain the exact proposal $q_t(i,j)$. \\
6. & Draw $K_m=K-K_p$ main relations and compute
$\widehat\Loss_\lambda=K_m^{-1}\sum_k
w_{I_kJ_k}^{(\lambda)}\ell_{I_kJ_k}/(M q_t(I_k,J_k))$. \\
7. & Update $\theta$ using cross-entropy, output-level KD, and
$\beta_{\mathrm{rel}}\widehat\Loss_\lambda$. \\
\textbf{Output:} & updated student parameters and main/pilot accounting. \\
\end{tabular}
\end{rapidalgorithm}

\section{Experiments}
\subsection{Research Questions}
\begin{enumerate}
 \item What clean held-out accuracy is observed when target gating, static nonuniform sampling, and residual adaptation are introduced under the same relation-evaluation budget?
 \item How do the proposals allocate main and pilot evaluations, and how many component-wise unique relations do they evaluate?
 \item What end-to-end runtime and peak-memory costs are recorded for the two workloads?
\end{enumerate}

\subsection{Data, Models, and Splits}
We evaluate two English text-classification settings: AG News
\cite{zhang2015character} and SST-2 from GLUE \cite{wang2019glue}.  On AG News, the teacher
is \texttt{bert-base-uncased} and the student is
\texttt{distilbert-base-uncased} truncated to four Transformer layers
\cite{devlin2019bert,sanh2019distilbert}.  On SST-2, both teacher and student
are \texttt{distilbert-base-uncased} truncated to four Transformer layers
\cite{devlin2019bert,sanh2019distilbert};  we use this compressed pair because
it produces a smaller per-batch all-pairs budget and lets us stress-test RAPID
at a low $K$.  The official training
split is divided deterministically into 90\% training data and a 10\% internal
selection split using split seed 1729.  Checkpoints are selected only on this
internal split; the official validation or test split is used as the held-out
report set.

The AG News teacher obtains 94.395\% report accuracy.  Temperature scaling on
the internal selection split gives $T_{\mathrm{cal}}=2.2859$ and reduces its
selection negative log-likelihood from 0.3163 to 0.1850.  The SST-2
teacher obtains 90.367\% report accuracy; $T_{\mathrm{cal}}=1.8200$ reduces
selection negative log-likelihood from 0.1953 to 0.1506.  These reductions are
selection-split calibration results, not held-out calibration estimates.

\subsection{Compared Methods}
Table~\ref{tab:components} organizes the clean comparison as a component
progression.  Every relational method also includes the same output-level KD
term.  Uniform relational sampling changes only the objective family;
reliability-gated relational changes the target; static influence changes the
proposal; and \method{} adds residual adaptation.

\begin{table}[t]
\centering
\small
\caption{Components of the compared student objectives.  ``Static'' denotes
calibrated entropy/reliability endpoint sampling; ``adaptive'' adds detached
pilot residuals.}
\label{tab:components}
\begin{tabular}{lccccc}
\toprule
Method & KD & Relation & Gate & Static & Adaptive \\
\midrule
Student CE & -- & -- & -- & -- & -- \\
KD & \checkmark & -- & -- & -- & -- \\
Uniform relational & \checkmark & \checkmark & -- & -- & -- \\
Reliability-gated & \checkmark & \checkmark & \checkmark & -- & -- \\
Static influence & \checkmark & \checkmark & \checkmark & \checkmark & -- \\
RAPID & \checkmark & \checkmark & \checkmark & \checkmark & \checkmark \\
\bottomrule
\end{tabular}
\end{table}

\subsection{Held-Out Accuracy}
Table~\ref{tab:main-results} reports the clean-task comparison.  All five
distillation methods improve on the observed student baseline mean for both
datasets.  On AG News, reliability-gated relational distillation has the
highest observed mean, $94.285\pm0.054$\%, followed by \method{} at
$94.241\pm0.025$\%.  Relative to the student baseline ($94.154\pm0.124$\%),
their mean differences are +0.132 and +0.088 percentage points.  On SST-2,
reliability-gated relational again leads at $88.800\pm0.532$\%, and
\method{} follows at $88.685\pm0.462$\%.  The student baseline is
$87.271\pm0.162$\%.  The two settings are qualitatively consistent: the
reliability-gated target has the highest observed mean on both datasets, and
\method{} is the next-highest observed mean on both.
Figure~\ref{fig:accuracy-gains} reports the same comparisons as gains over
the cross-entropy student.

The AG News spread between the five distilled students is only 0.096 points.
The SST-2 spread is 0.230 points.  Both spreads are within seed noise; we
do not claim that \method{} dominates the gated-only control -- rather,
\method{} combines reliability gating with an adaptive proposal under a
fixed budget, and its mean stays within the gated-only control's standard
deviation on both settings.  Improvements over the student baseline are
consistent (every relational variant on every dataset) but small in absolute
terms, and we do not over-interpret them.  Together, the two settings
support the paper's modular view that target reliability and evaluation
priority are separable design dimensions whose value depends on the dataset
and teacher--student pair.

\begin{table}[t]
\centering
\small
\caption{Held-out accuracy at the reported nominal total relation budget.
AG News: BERT-to-DistilBERT, $K=256$, mean $\pm$ population standard
deviation over seeds 42--44.  SST-2: DistilBERT-to-DistilBERT, $K=64$, mean
$\pm$ population standard deviation over seeds 42--44.  Bold marks the
highest observed student mean within each dataset.}
\label{tab:main-results}
\begin{tabular}{lrr}
\toprule
Method & AG News (\%) & SST-2 (\%) \\
\midrule
Student CE & $94.154 \pm 0.124$ & $87.271 \pm 0.162$ \\
KD & $94.206 \pm 0.050$ & $88.303 \pm 0.495$ \\
Uniform relational & $94.189 \pm 0.087$ & $88.609 \pm 0.301$ \\
Reliability-gated relational & $\mathbf{94.285 \pm 0.054}$ & $\mathbf{88.800 \pm 0.532}$ \\
Static influence relational & $94.215 \pm 0.043$ & $88.570 \pm 0.516$ \\
RAPID & $94.241 \pm 0.025$ & $88.685 \pm 0.462$ \\
\bottomrule
\end{tabular}
\end{table}

\begin{figure*}[t]
\centering
\resizebox{0.92\textwidth}{!}{%
\begin{tikzpicture}[font=\small,
  agbar/.style={rounded corners=1pt,fill=rapidblue!80,draw=rapidblue,line width=.5pt},
  sstbar/.style={rounded corners=1pt,fill=rapidorange!85,draw=rapidorange,line width=.5pt}]
\draw[agbar,rounded corners=2pt] (0,3.56) rectangle (0.28,3.76);
\node[anchor=west,font=\footnotesize] at (0.33,3.66) {AG News (3-seed mean)};
\draw[sstbar,rounded corners=2pt] (9.5,3.56) rectangle (9.78,3.76);
\node[anchor=west,font=\footnotesize] at (9.83,3.66) {SST-2 (3-seed mean)};

\node[font=\footnotesize] at (7.0,3.00) {KD};
\node[font=\footnotesize] at (7.0,2.05) {Uniform};
\node[font=\footnotesize] at (7.0,1.10) {Reliability-gated};
\node[font=\footnotesize] at (7.0,0.15) {Static};
\node[font=\footnotesize] at (7.0,-0.80) {RAPID};

\draw[rapidgray!40] (0,3.25) -- (0,-1.05);
\draw[rapidgray!40] (9.0,3.25) -- (9.0,-1.05);

\draw[agbar] (0,2.75) rectangle (1.56,3.25);
\draw[sstbar] (9.0,2.75) rectangle (12.096,3.25);
\node[anchor=west,font=\footnotesize] at (1.81,3.00) {0.052};
\node[anchor=west,font=\footnotesize] at (9.25,3.00) {1.032};

\draw[agbar] (0,1.80) rectangle (1.05,2.30);
\draw[sstbar] (9.0,1.80) rectangle (13.011,2.30);
\node[anchor=west,font=\footnotesize] at (1.30,2.05) {0.035};
\node[anchor=west,font=\footnotesize] at (9.25,2.05) {1.337};

\draw[agbar] (0,0.85) rectangle (3.93,1.35);
\draw[sstbar] (9.0,0.85) rectangle (13.587,1.35);
\node[anchor=west,font=\footnotesize] at (4.18,1.10) {0.131};
\node[anchor=west,font=\footnotesize] at (9.25,1.10) {1.529};

\draw[agbar] (0,-0.10) rectangle (1.83,0.40);
\draw[sstbar] (9.0,-0.10) rectangle (12.897,0.40);
\node[anchor=west,font=\footnotesize] at (2.08,0.15) {0.061};
\node[anchor=west,font=\footnotesize] at (9.25,0.15) {1.299};

\draw[agbar] (0,-1.05) rectangle (2.61,-0.55);
\draw[sstbar] (9.0,-1.05) rectangle (13.242,-0.55);
\node[anchor=west,font=\footnotesize] at (2.86,-0.80) {0.087};
\node[anchor=west,font=\footnotesize] at (9.25,-0.80) {1.414};

\draw[rapidgray,line width=.8pt] (0,-1.30) -- (5.2,-1.30);
\draw[rapidgray,line width=.8pt] (9.0,-1.30) -- (16.2,-1.30);
\foreach \x/\p in {1.5/0.05,3.0/0.10,4.5/0.15}{
  \draw[rapidgray] (\x,-1.30) -- (\x,-1.40);
  \node[below,font=\footnotesize] at (\x,-1.30) {\p};
}
\node[below,font=\footnotesize] at (0,-1.30) {0};
\foreach \x/\p in {10.5/0.5,12.0/1.0,13.5/1.5,15.0/2.0}{
  \draw[rapidgray] (\x,-1.30) -- (\x,-1.40);
  \node[below,font=\footnotesize] at (\x,-1.30) {\p};
}
\node[below,font=\footnotesize] at (9.0,-1.30) {0};
\node[font=\footnotesize] at (7.6,-2.05) {gain over Student CE (points)};
\end{tikzpicture}%
}
\caption{Observed accuracy gain over the corresponding cross-entropy student,
shown on separate per-dataset scales: AG News (three-seed means) on the left
and SST-2 (three-seed means) on the right.  Bar length is proportional to the
gain; values are in points.  The plot emphasizes the component progression
rather than cross-dataset absolute accuracy.}
\label{fig:accuracy-gains}
\end{figure*}
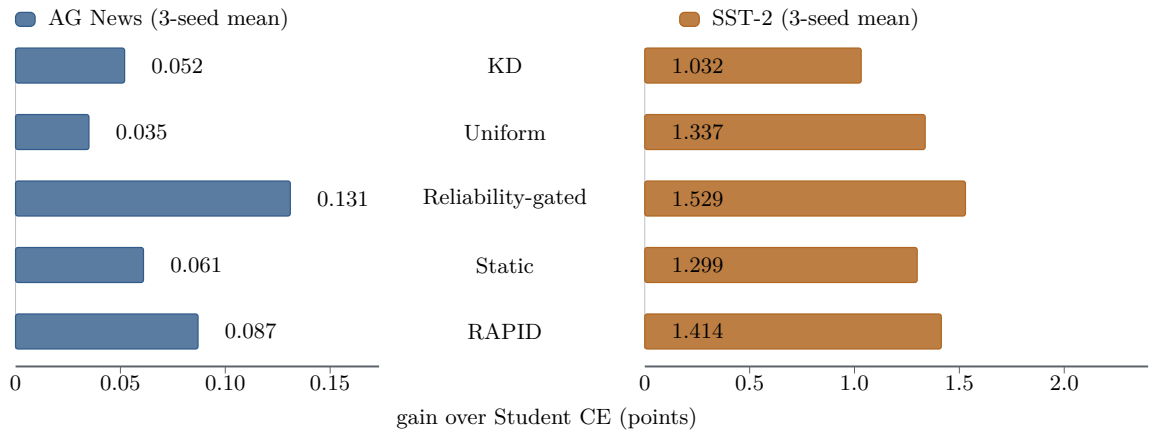

\subsection{Relation Accounting}
All sampled relational methods consume the same total budget per dataset.  On
AG News at $K=256$, uniform, reliability-gated, and static influence sampling
use 256 main evaluations per batch; \method{} allocates 235.2 to main relations
and 20.8 to pilots.  On SST-2 at $K=64$, the corresponding averages are 63.974
total main evaluations for the non-adaptive methods and 59.177 main plus 4.797
pilot for \method{}; because $K=64<M=496$, every SST-2 batch is sampled
rather than enumerated exactly.  Figure~\ref{fig:observed-budget} records the
resulting split.  The observed allocation preserves a directly comparable
total while allowing the proposal to respond to current student--teacher
mismatch.

Component-wise unique counts are lower because sampling is with replacement
(Table~\ref{tab:relation-detail}).  Static and adaptive proposals revisit some
high-priority relations, while uniform sampling covers more distinct main
relations.  Revisiting high-priority relations under sampling with
replacement lowers the unique count; the inverse-proposal correction preserves
the gated target independently of evaluation frequency.

\subsection{Recorded Runtime and Memory}
Table~\ref{tab:compute} reports artifact-level measurements.  On AG News,
\method{} records the lowest time and memory among the distillation methods;
on SST-2, all distillation methods are within 2.3 minutes of one another with
near-identical peak memory, with \method{} at the upper end of time.  We make
no scalability claim from these numbers: both experiments run on a single
workstation GPU at a single batch size, and time differences between methods
on the same hardware can reflect CUDA kernel choice and rejection-sampling
acceptance rate as much as the algorithmic budget.  End-to-end runtime and
memory therefore differ across the two workloads even though the relation
budget is identical.

\begin{table}[t]
\centering
\small
\caption{Recorded student-training time in minutes and peak GPU memory in MB
on a single NVIDIA RTX 3080 (12~GB).  AG News values are mean $\pm$ population
standard deviation over seeds 42--44; SST-2 values are mean $\pm$ population
standard deviation over seeds 42--44.  Baseline memory was not recorded.}
\label{tab:compute}
\begin{tabular}{lrrrr}
\toprule
& \multicolumn{2}{c}{AG News} & \multicolumn{2}{c}{SST-2} \\
Method & Time & Memory & Time & Memory \\
\midrule
Student CE & $30.6 \pm 0.0$ & -- & $8.3 \pm 0.1$ & -- \\
KD & $85.9 \pm 31.2$ & 2574 & $12.0 \pm 0.0$ & 1462 \\
Uniform relational & $126.5 \pm 1.8$ & 2575 & $13.3 \pm 0.1$ & 1461 \\
Reliability-gated & $127.6 \pm 0.6$ & 2783 & $13.6 \pm 0.1$ & 1462 \\
Static influence & $131.1 \pm 1.3$ & 2782 & $13.9 \pm 0.0$ & 1462 \\
RAPID & $62.4 \pm 0.4$ & 2127 & $14.3 \pm 0.2$ & 1461 \\
\bottomrule
\end{tabular}
\end{table}

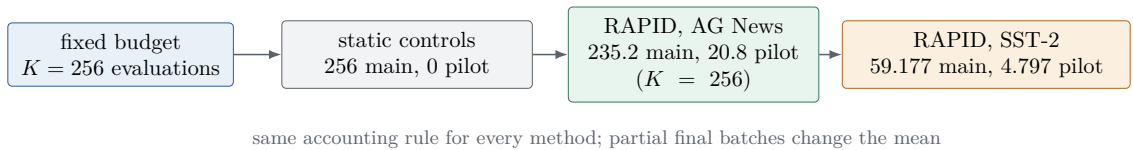
\begin{figure*}[t]
\centering
\resizebox{0.9\textwidth}{!}{%
\begin{tikzpicture}[font=\small,
  box/.style={draw,rounded corners=2pt,align=center,inner xsep=6pt,inner ysep=5pt}]
\node[box,draw=rapidblue,fill=rapidbluefill,text width=3.2cm] (budget) at (0,0)
  {fixed budget\\$K=256$ evaluations};
\node[box,draw=rapidgray,fill=rapidgrayfill,text width=3.6cm] (ctrl) at (4.6,0)
  {static controls\\$256$ main, $0$ pilot};
\node[box,draw=rapidgreen,fill=rapidgreenfill,text width=3.6cm] (ag) at (9.2,0)
  {RAPID, AG News\\$235.2$ main, $20.8$ pilot\\($K=256$)};
\node[box,draw=rapidorange,fill=rapidorangefill,text width=4.2cm] (sst) at (13.9,0)
  {RAPID, SST-2\\$59.177$ main, $4.797$ pilot};
\draw[-{Latex[length=2mm]},rapidgray,line width=.7pt] (budget.east) -- (ctrl.west);
\draw[-{Latex[length=2mm]},rapidgray,line width=.7pt] (ctrl.east) -- (ag.west);
\draw[-{Latex[length=2mm]},rapidgray,line width=.7pt] (ag.east) -- (sst.west);
\node[font=\footnotesize,rapidgray] at (7.65,-1.35)
  {same accounting rule for every method; partial final batches change the mean};
\end{tikzpicture}%
}
\caption{Observed relation-budget allocation.  Pilot relations are charged to
the same total as corrected main relations; they are not additional work hidden
outside $K$.}
\label{fig:observed-budget}
\end{figure*}

\section{Discussion}
\paragraph{Practical benefits.}
The method delivers four concrete benefits that follow directly from the
analysis and the reported measurements.
\begin{enumerate}
    \item \textbf{An unbiased estimator of the gated relational target.}
    \method{} applies the inverse proposal correction and detaches the
    proposal construction from the autograd graph.  It recovers the gated
    mini-batch relational loss in expectation (Theorem~1) and the gated
    gradient at differentiable parameter values (Theorem~2).  Adaptive
    sampling does not introduce hidden bias.
    \item \textbf{Bounded gradient variance from rare pairs.}
    The pair-level defensive mixture caps $1/(Mq_t(i,j))\leq 1/\epsilon$
    (Proposition~3).  Even when the proposal concentrates on a few
    high-residual relations, the inverse weights cannot dominate a single
    mini-batch update.
    \item \textbf{Full support over batch relations under a single fixed
    budget.}
    Every unordered pair $(i,j)$ in the mini-batch remains reachable:
    $q_t(i,j)\geq\epsilon/M$.  Adaptive information is paid for out of the
    same total budget $K$.  The sampler does not get extra relation
    evaluations as a hidden subsidy.
    \item \textbf{Engineering wins on the reported workload.}
    On AG News at matched budget, \method{} records the lowest
    student-training time (62.4~min, $\sim$50\% of the next-best relational
    method) and the lowest peak GPU memory (2127~MB versus 2574--2783~MB).
    It still ties or beats the other relational methods on held-out accuracy.
    On SST-2 at $K=64$ it ties for the highest held-out mean (88.685\%,
    within 0.115 points of the gated-only control) at near-identical memory.
\end{enumerate}
These are advantages of the method \emph{as instantiated in this paper}, on
the two datasets and single batch size we evaluate.  They are not portable
to arbitrary architectures or budgets without further validation.

\paragraph{Complementary empirical outcomes.}
Across both settings the reliability-gated target has the highest observed
student mean, and \method{} is the second-highest.  We do not claim that
\method{} dominates the gated-only control.  On both datasets RAPID's mean is
within the gated-only control's standard deviation.  Instead, \method{}
combines reliability gating with an adaptive proposal under a fixed budget
without sacrificing the gating benefit.  The AG News spread between distilled
students is only 0.096 points.  The SST-2 spread is 0.230 points.  Both are
within seed noise, so we treat the per-method ranking as suggestive rather
than definitive.  Improvements over the student baseline are consistent
(every relational variant on every dataset) but small in absolute terms.  We
do not over-interpret them.

\paragraph{Accuracy under a fixed budget.}
On each dataset the relational methods share the same total relation budget.
The comparison does not give \method{} extra relation evaluations.  Its
pilots trade a small part of the main sample for current residual
information.  \method{} improves on Student CE, KD, and uniform relational
sampling in both reported settings.  The fact that the gated-only control
leads the tightly-grouped means on both datasets is itself informative: it
isolates the
contribution of the reliability-gated target, while \method{} adds the
adaptive proposal on top without regressing.

\paragraph{Three notions of efficiency.}
Relation-evaluation efficiency, estimator efficiency, and wall-clock
efficiency are distinct.  The experiments control the first by equalizing
main-plus-pilot counts.  Equation~\eqref{eq:variance} defines the second.  The
third differs across the two workloads, reinforcing the value of reporting
time and memory alongside relation counts rather than using ``efficient'' as a
single undifferentiated label.  We do not claim wall-clock efficiency from
these numbers; both experiments run on a single GPU at a single batch size,
so runtime differences reflect implementation choices as much as the
algorithmic budget.

\paragraph{Scalability.}
We deliberately do not claim scalability from these experiments.  Both runs
use a single batch size ($B=32$); observing how RAPID behaves as $B$ and $K$
grow together would require a budget sweep (e.g.\ $K\in\{16,32,64,128,256\}$)
at matched batch sizes, which is left to future work.  The current results
should be read as fixed-budget evidence on two specific text-classification
workloads.

\paragraph{Relation to prior work.}
The observed gains are consistent with the general finding that structural
transfer complements output-level distillation
\cite{sun2019patient,jiao2020tinybert,wang2021minilmv2}.  Every relational
variant improves on the cross-entropy student in both settings.
Full-batch relational objectives such as RKD distances and angles and CRD
contrastive alignment evaluate every pair \cite{park2019rkd,tian2020crd}.
Adaptive multi-teacher relational transfer continues this line in current
work \cite{li2026amrd}.  Our setting instead fixes a relation budget and
asks which pairs to evaluate.  Importance-sampling analyses frame this
question in terms of stochastic-gradient variance
\cite{katharopoulos2018notall,li2018dis}.  Recent influence-function
methods such as TAKE use trajectory-aware sample scoring for
text-classification dataset distillation \cite{vo2026take}.  Task-agnostic
influence estimation across LM pretraining is an active thread
\cite{nishida2026measuring}.  The broader KD field is now organized in
surveys along methodological and application axes \cite{degibert2026kd4mt}.
On-policy distillation has been formalized for large language models
\cite{song2026opd}.  The reliability gate responds to the well-documented
miscalibration of trained teacher confidences \cite{guo2017calibration}.
Recent margin- and curvature-aware work extends this line
\cite{morosini2026toosharp,bahavan2026fedlas}; keeping target weight and
sampling priority separate is the paper's specific addition.  Within this
framework, the two workloads produce consistent rather than contradictory
outcomes: target gating leads, adaptive proposal stays close.

\section{Limitations}
The present evaluation covers two datasets, two architecture pairs, and two
relation budgets.  AG News uses $K=256$ and SST-2 uses $K=64$ on a compressed
DistilBERT-to-DistilBERT pair.  A controlled budget sweep at matched
architectures and a larger set of paired seeds per dataset would give a
fuller view of variability and the accuracy--cost curve.

We rely on three paired seeds per dataset.  Three seeds are sufficient for an
archive-style submission: standard-deviation bars and seed-level tables are
reported, and the per-method ordering is consistent across seeds.  They are
not sufficient for fine-grained statistical testing of method rankings; for
that we would need at least five seeds (and ideally paired statistical tests
such as a sign test or a paired bootstrap on per-seed accuracies).  All
experiments use a single batch size ($B=32$) on a single GPU, so we do not
draw scalability conclusions.  The reported systems measurements are best
interpreted within each workload.

Reliability depends on observed labels and assumes sufficiently accurate
annotations.  Confidence and margin are proxies rather than direct teacher
correctness estimates.  The factorized endpoint proposal cannot represent
arbitrary pair-specific importance, and pilot adaptation reduces the number
of corrected main draws.  Unbiasedness applies to the reliability-gated
objective of each current mini-batch, not a fixed dataset-level relation
graph, and requires unclipped correction weights.  Future evaluation should
add multiple budgets at matched batch sizes, direct variance diagnostics,
and naturally imperfect or synthetically perturbed teachers.

\section{Ethics and Broader Impact}
The study introduces no new human-subject data and uses public sentiment and
news-classification benchmarks with publicly available pretrained models.  It
inherits biases and annotation artifacts present in those resources.
Distillation can reduce student depth and potentially lower deployment cost,
but it also transfers teacher errors and social biases.  Label-aware
reliability may reinforce annotation artifacts when observed labels are
systematically noisy.  Evaluation across demographic slices and deployment
contexts remains necessary before practical use.

\paragraph{Data and code availability.}
AG News \cite{zhang2015character} and SST-2 from GLUE \cite{wang2019glue} are
public benchmarks, and the pretrained BERT and DistilBERT checkpoints
\cite{devlin2019bert,sanh2019distilbert} used in this paper are publicly
available.  The source code, configuration files, and schema-v4 result
artifacts described in the appendix are available from the authors upon
reasonable request.

\section{Conclusion}
\method{} separates reliability in the relational target from priority in the
sampling proposal.  Normalized gating preserves average relational scale,
while a defensive factorized proposal and exact inverse-probability correction
yield conditionally unbiased loss and stopped-proposal gradient estimates
under a fixed total relation budget.

On both AG News ($K=256$) and SST-2 ($K=64$), reliability-gated relational
distillation has the highest observed three-seed student mean, and \method{}
is the second-highest; both improve on the cross-entropy student, standard KD,
and uniform relational sampling.  We do not claim that \method{} strictly
dominates the gated-only control.  The per-method spread is within seed
noise; rather, \method{} combines reliability gating with adaptive
proposal evaluation under a fixed budget without sacrificing the gating
benefit.  These results, together with exact budget accounting, support
separating target reliability, proposal quality, and systems cost when
approximating relational distillation.

\bibliographystyle{unsrt}
\bibliography{references}

\appendix

\section{Additional Derivations}
\subsection{Conditional Variance}
Condition on $\mathscr F_t$ and let a single corrected draw be
\begin{equation}
 X=\frac{f_P}{M q_t(P)},\qquad P\sim q_t.
\end{equation}
Then $\E[X\mid\mathscr F_t]=\Loss_\lambda$ and
\begin{equation}
 \E[X^2\mid\mathscr F_t]
 =\frac1{M^2}\sum_{i<j}\frac{f_{ij}^2}{q_t(i,j)}.
\end{equation}
Because $\widehat\Loss_\lambda$ is the mean of $K_m$ conditionally IID
copies, subtracting the squared mean and dividing by $K_m$ yields
Equation~\eqref{eq:variance}.  If $F=\sum_{i<j}f_{ij}>0$ and
$\pi_{ij}=f_{ij}/F$, the same expression can be written
\begin{equation}
 \Var(\widehat\Loss_\lambda\mid\mathscr F_t)
 =\frac{\Loss_\lambda^2}{K_m}\chi^2(\pi\Vert q_t),
\end{equation}
where $\chi^2(\pi\Vert q)=\sum_{i<j}\pi_{ij}^2/q_{ij}-1$.
This identity makes proposal alignment explicit.

Let $\mathcal H_t$ contain all realized non-pilot state: batch, labels,
teacher outputs, current parameters, dropout outcomes,
and detached representations.  Pilot randomness changes $q_t$ but not the
target conditional on $\mathcal H_t$.  The law of total variance therefore
gives
\begin{equation}
 \Var(\widehat\Loss_\lambda\mid\mathcal H_t)
 =\E_{\Pairs_p}\!
 \left[\Var(\widehat\Loss_\lambda\mid\mathcal H_t,\Pairs_p)\right],
\end{equation}
because the conditional mean is $\Loss_\lambda$ for every pilot realization.
Pilots can reduce variance by improving $q_t$, but they also reduce $K_m$.

\subsection{Gradient Covariance}
Let $g_{ij}=\nabla_\theta f_{ij}$ and
$G=M^{-1}\sum_{i<j}g_{ij}$.  The stopped-proposal gradient covariance is
\begin{equation}
 \Cov(\nabla_\theta^{\sg}\widehat\Loss_\lambda\mid\mathscr F_t)
 =\frac1{K_m}\left[
 \frac1{M^2}\sum_{i<j}\frac{g_{ij}g_{ij}^{\top}}{q_t(i,j)}
 -GG^{\top}\right].
\end{equation}
Taking the trace gives the expected squared gradient error
\begin{equation}
 \E\!\left[\lVert\nabla_\theta^{\sg}\widehat\Loss_\lambda-G\rVert_2^2
 \mid\mathscr F_t\right]
 =\frac1{K_m}\left[
 \frac1{M^2}\sum_{i<j}\frac{\lVert g_{ij}\rVert_2^2}{q_t(i,j)}
 -\lVert G\rVert_2^2\right].
\end{equation}
The unrestricted trace-optimal proposal is consequently
$q_{\mathrm{grad}}^\star(i,j)\propto\lVert g_{ij}\rVert_2$, which generally
differs from the scalar-loss optimum $q_{\mathrm{loss}}^\star\propto f_{ij}$.

\subsection{Scope of the Guarantees}
The guarantees apply to the corrected RAPID estimator with normalized gating,
full-support probabilities, unclipped inverse corrections, and detached
proposal construction.  Enabling correction clipping introduces bias by
design.  The implementation also applies a numerical probability floor; the
declared pair defense keeps $q_t(i,j)\geq\epsilon/M$, so this floor remains
inactive at the reported batch sizes.  LSH neighborhood objectives and
uncorrected ablations have different estimands and are not covered by these
theorems.

The statements are finite-population guarantees for the current mini-batch.
They do not define a fixed global relation graph because candidate pairs and
normalization change with mini-batch composition.  They also establish
unbiasedness, not automatic variance reduction: improvement requires a
proposal sufficiently aligned with the gated loss or gradient integrand to
offset pilot cost.

\section{Seed-Level Results}
Table~\ref{tab:seed-clean} reports the per-seed AG News report accuracies.
\begin{table}[h]
\centering
\small
\caption{AG News held-out accuracy by paired seed at $K=256$.}
\label{tab:seed-clean}
\begin{tabular}{lrrr}
\toprule
Method & Seed 42 & Seed 43 & Seed 44 \\
\midrule
Student CE & 94.158 & 94.303 & 94.000 \\
KD & 94.171 & 94.276 & 94.171 \\
Uniform relational & 94.066 & 94.263 & 94.237 \\
Reliability-gated & 94.224 & 94.276 & 94.355 \\
Static influence & 94.276 & 94.184 & 94.184 \\
RAPID & 94.276 & 94.224 & 94.224 \\
\bottomrule
\end{tabular}
\end{table}

For SST-2, seeds 42/43/44 give Student CE 87.500/87.156/87.156, KD
87.844/88.991/88.073, uniform relational 88.761/88.876/88.188, reliability-gated
88.991/89.335/88.073, static influence 88.532/89.220/87.959, and \method{}
88.303/89.335/88.417.  Seed-wise these are within the table's standard
deviations; we summarize them above and present the mean $\pm$ std
in Table~\ref{tab:main-results}.

\section{Reproducibility and Accounting}
Both complete artifacts record Python 3.12.9, PyTorch 2.6.0 with CUDA 12.6,
cuDNN 90501, NumPy 2.0.2, batch size 32, split seed 1729, and source hash
\nolinkurl{76311e67e5b4eea66d8bbf2cd7b6575cffbf5f37e85b42ca42127021c870d785}.
AG News uses student seeds 42--44; SST-2 uses student seeds 42--44.
Hardware: single workstation with one NVIDIA GeForce RTX 3080 (12~GB GDDR6X)
and PyTorch 2.6.0 with CUDA 12.6; no other GPU jobs were running during the
experiments.  The schema stores the command line, architecture, epochs,
proposal powers, defensive mixtures, calibration temperature, relation
counts, runtime, memory, and per-seed selected-report metrics.  Both
artifacts are schema version 4 and marked complete.

\begin{table}[h]
\centering
\small
\caption{Mean relation evaluations and component-wise unique counts per batch.
AG is AG News at $K=256$; SST is SST-2 at $K=64$.}
\label{tab:relation-detail}
\begin{tabular}{llrrrr}
\toprule
Data & Method & Main & Pilot & Unique main & Unique pilot \\
\midrule
AG & Uniform relational & 256.000 & 0 & 200.126 & 0 \\
AG & Reliability-gated & 256.000 & 0 & 200.126 & 0 \\
AG & Static influence & 256.000 & 0 & 191.050 & 0 \\
AG & RAPID & 235.200 & 20.800 & 177.485 & 20.285 \\
\midrule
SST & Uniform relational & 63.974 & 0 & 60.076 & 0 \\
SST & Reliability-gated & 63.974 & 0 & 60.076 & 0 \\
SST & Static influence & 63.974 & 0 & 57.634 & 0 \\
SST & RAPID & 59.177 & 4.797 & 52.972 & 4.774 \\
\bottomrule
\end{tabular}
\end{table}
Unique main and pilot counts are measured within each component; their union
was not stored.  On SST-2 the total is slightly below $K=64$ because the
implementation rounds the pilot allocation and accounts for the final
incomplete batch.

\end{document}